\documentclass[letterpaper]{article}
\usepackage{aaai}
\usepackage{times}
\usepackage{helvet}
\usepackage{courier}
\usepackage{amsmath,amsthm,amssymb,mathtools}
\usepackage{algorithm,algorithmicx,algpseudocode}

\newcommand{\defword}[1]{\textbf{\boldmath{#1}}}
\newcommand{\argmin}{\operatornamewithlimits{argmin}}
\newcommand{\argmax}{\operatornamewithlimits{argmax}}

\newtheorem{theorem}{Theorem}
\newtheorem{definition}{Definition}

\begin{document}
\title{A Unified View of Large-scale Zero-sum Equilibrium Computation}
\author{Kevin Waugh\\
{\tt waugh@cs.cmu.edu}\\
Department of Computer Science\\
Carnegie Mellon University\\
5000 Forbes Ave, Pittsburgh, PA, 15213 USA
\And
J. Andrew Bagnell\\
{\tt dbagnell@ri.cmu.edu}\\
The Robotics Institute\\
Carnegie Mellon University\\
5000 Forbes Ave, Pittsburgh, PA, 15213 USA
}
\maketitle
\begin{abstract}
\begin{quote}

The task of computing approximate Nash equilibria in large zero-sum
extensive-form games has received a tremendous amount of attention due mainly
to the Annual Computer Poker Competition.  Immediately after its inception, two
competing and seemingly different approaches emerged---one an application of
no-regret online learning, the other a sophisticated gradient method applied to
a convex-concave saddle-point formulation.  Since then, both approaches have
grown in relative isolation with advancements on one side not effecting the
other.  In this paper, we rectify this by dissecting and, in a sense, unify the
two views.

\end{quote}
\end{abstract}

\section{Introduction}

%
%

The first annual Computer Poker Competition was held in 2007 providing a
testbed for adversarial decision-making with imperfect information.  Though
incredible advancements have been made since then, the solution of an abstract
game still remains a critical component of the top agents.  The strength of
such a strategy correlates with how well the abstract game models the
intractably large full game.  Algorithmic improvements in equilibrium
computation enable teams to solve larger abstract games and thus improve the
overall strength of their agents.

The first agents used off-the-shelf linear programming packages to solve
the abstract games.  Shortly after, two specialized equilibrium-finding
techniques emerged drastically reducing resource requirements by allowing
implicit game representations.  One method, counterfactual regret minimization
(CFR), combines many simple no-regret learners together to minimize overall
regret and as a whole they converge to an equilibrium~\cite{zinkevich08}.  The
other, an application of Nesterov's excessive gap technique (EGT,
\citeyear{nesterov05}), is a gradient method that attacks the convex-concave
saddle-point formulation directly~\cite{gilpin07}.

Currently, the two techniques are thought of as simply different and competing.
Though both improved, the advancements and the methods remained isolated.
At this point, CFR has more widespread adoption due to its simplicity and
the power of the sampling schemes available to it.

In this paper we connect CFR and EGT.  Specifically, both define Bregman
divergences with the same structure.  This viewpoint allows us to make 
important connections between the two as well as current
research on convex optimization and no-regret learning.  In particular, we show
that CFR can be thought of as smoothed fictitious play and its dual weights are
a function of the opponent's average strategy; with the appropriate step-size
the primal iterate converges to the solution (as opposed to the average); and
that a convergence rate of $O(1/T)$ can be achieved while sampling.

\section{Zero-Sum Extensive-form Games}

A \defword{extensive-form game} is a tuple $\Gamma=(N,\mathcal
H,p,\sigma_c,\mathcal I,u)$ (see, {\em e.g.}, \citeauthor{osborne94}).  The
game has $N$ players, the set of which is $[N]$.  The set of
\defword{histories}, $\mathcal H$, form a tree rooted at $\phi$, the empty
history.  For a history $h\in\mathcal H$, we denote the set of available
\defword{actions} by $A(h)$.  For any action $a\in A(h)$, the history $ha
\in\mathcal H$ is a child of $h$.  The tree's leaves, or \defword{terminal}
histories, are denoted by $\mathcal Z\subseteq \mathcal H$.  At a
\defword{non-terminal} history the \defword{player choice function}, $p :
\mathcal H\setminus\mathcal Z\rightarrow \mathcal [N]\cup\{c\}$, determines who
is to act, either a player or nature.  Nature's policy is denoted $\sigma_c$
defines a distribution over actions when it is to act $\sigma_c(\cdot|h) \in
\Delta_{A(h)}$.  The \defword{information partition}, $\mathcal
I=\cup_{i\in[N]}\mathcal I_i$, is a partition the players' histories.  All
histories in an \defword{information set} are indistinguishable to the player
to act.  We have $\forall I\in\mathcal I, h,h'\in I$, that $p(h) = p(h')$ and
$A(h) = A(h')$.  Finally, at a terminal history, $z\in\mathcal Z$, the
\defword{utility for player $i$}, $u_i : \mathcal Z \rightarrow
\mathbb R$, determines the reward for player $i$ reaching terminal
$z$.

Each player plays the game by means of a \defword{behavioral strategy}.  A
behavioral strategy for player $i$, $\sigma_i\in\Sigma_i$, maps histories to
distributions over actions, $\sigma_i(\cdot|h) \in\Delta_{A(h)}$.  When the
player is to act, their choice is drawn from this distribution.  A strategy must
respect the information partition, $\forall I\in\mathcal I_i, h,h'\in I,
\sigma_i(\cdot|h) = \sigma_i(\cdot|h')$.  We call the tuple of $N$ strategies,
$(\sigma_1,\ldots,\sigma_N)$, a \defword{strategy profile}.

There are two additional properties we require of the extensive-form games
we examine.  First, we consider two-player \defword{zero-sum} games.  That is,
$N = 2$ and $u_2(z) = -u_1(z)$; what one player wins the other loses.

Second, we will consider games of \defword{perfect recall}---neither player 
is forced forget any information they once knew.  Mathematically this
requirement is that all histories in an information set share the same sequence
of information sets and actions from the point-of-view of the acting player.

With these additional restrictions, we can conveniently represent a game in its
\defword{sequence form}, $\Gamma=(A,E,F)$.  Following from perfect recall, any
sequence of information set/action pairs, henceforth simply sequence, is
uniquely identified by its final element.  Consequently, we can represent
behavioral strategies as vectors, called \defword{realization plans}, such that
the expected utility of the game is a bilinear product $x^TAy$.  In particular,
a realization plan for the row player is a non-negative vector $x$ indexed by
sequences such that $\sum_{a\in A(I)} x(I, a) = x(\text{parent}(I))$ for all
$I\in\mathcal I_i$ and where $x(\phi) = 1$.  In words, the probability mass
flowing out of an information set equals the probability of playing to reach
that information set.  The constraint on the empty sequence, $\phi$, normalizes
the probability.  We represent these linear equality constraints with the
matrix $E$, {\em i.e.,} $Ex = e_1$, and thus $\Sigma_1 = \{x \mid Ex = e, x\ge
0\}$.  For the column player, we have corresponding concepts $y$, $F$, and
$\Sigma_2$.  The matrix $A$ is the \defword{payoff matrix}.  Entry $a_{i,j}$ is
the expected payoff to the row player over all terminals reach by sequences $i$
and $j$.

A pair of strategies, $(x,y)$ is said to be an \defword{$\varepsilon$-Nash
equilibrium} if neither player can benefit by more than $\varepsilon$ by
deviating to another strategy.  In particular,
\begin{align*}
	x^TAy + \varepsilon & \ge x'Ay, \mbox{~and} & \forall x'\in\Sigma_1\\
	-x^TAy + \varepsilon & \ge -x^TAy'. & \forall y'\in\Sigma_2
\end{align*}
Remarkably, a Nash equilibrium always exists and in our case we can
efficiently find $\varepsilon$-equilibria.  In the next sections
we will discuss and relate large-scale methods to do so.

\subsection{Counterfactual Regret Minimization}

Online learning is a powerful framework for analyzing the performance of
adaptive algorithms.  At time $t\in[T]$, an online algorithm chooses a
policy $x^t\in\Sigma$ and then receives reward vector $u^t\in\mathcal K$.  It
aims to minimize \defword{external regret},
\begin{align*}
\max_{x^*\in\Sigma} \sum_{t=1}^T u^t\cdot x^* - u^t\cdot x^t,
\end{align*}
its utility relative to the best fixed policy in hindsight.  An algorithm is
said to be \defword{no-regret} if its regret grows sublinear in $T$ for any
sequence of $u$'s from bounded set $\mathcal K$.  That is, if the bound on its
time-averaged regret approaches zero in the worst-case~\cite{cesa-bianchi06}. 

There is an important connection between no-regret learning and zero-sum
equilibrium computation.  Two no-regret algorithms in self-play converge
to a Nash equilibrium.  Operationally, the row player gets reward 
reward $u^t = Ay^{t}$ and the column player $u^t = -A^Tx^{t}$.
\begin{theorem}
If two no-regret algorithms play a zero-sum game against one and other for $T$
iterations and have time-averaged regret less than $\varepsilon$, then their
average strategies $(\bar x, \bar y)$ are a $2\varepsilon$-Nash equilibrium.  Here
$\bar x = \sum_{t=1}^T x^t / T$.
\end{theorem}
\begin{proof}
For any $x'\in\Sigma_1$ and $y'\in\Sigma_2$,
\begin{align*}
\frac{1}{T} \sum_{t=1}^T x'\cdot Ay^t - x^t\cdot Ay^t & \le \varepsilon, ~\mbox{~and}\\
\frac{1}{T} \sum_{t=1}^T (-x^t\cdot Ay') - (-x^t\cdot Ay^t) & \le \varepsilon
\shortintertext{adding the two inequalities together}
\frac{1}{T} \sum_{t=1}^T x'\cdot Ay^t - x^t\cdot Ay' & \le 2\varepsilon \\
\shortintertext{substituting in the definitions of $\bar{x}$ and $\bar{y}$}
x'\cdot A\bar{y} - \bar{x}\cdot Ay' & \le 2\varepsilon\\
\shortintertext{choosing $y' = \bar{y}$, we get for all $x'\in\Sigma_1$}
x'\cdot A\bar{y} & \le \bar{x}\cdot A\bar{y} + 2\varepsilon
\end{align*}
Similarly, if instead we choose $x' = \bar{x}$, we get the second
inequality in the definition of a $2\varepsilon$-Nash.
\end{proof}

The space of mixed strategies, $\Sigma_1$ and $\Sigma_2$, though structured,
is complicated.  \citeauthor{zinkevich08} overcome this and describe a
no-regret algorithm over the space of realization plans.  Their approach
minimizes a new notion of regret---counterfactual regret---at each information
set using simple no-regret algorithms for the probability simplex.  They show
that counterfactual regret bounds external regret, thus their approach computes
an equilibrium in self-play.

The \defword{counterfactual utility} for action $a$ at information set $I$ is
the expected utility given that the player tries to and successfully plays
action $a$.  That is, we weight a terminal by the total probability of the
opponent and chance reaching it, but only by the remaining probability for the
acting player.  Computing counterfactual utilities is done by traversing the
game tree, or one sparse matrix-vector product.

There are numerous no-regret algorithms operating over the probability simplex,
$\Sigma = \Delta$.  Let us present two.  The first,
regret-matching~\cite{hart00}, is also known as the polynomially-weighted
forecaster.  It is the most commonly used algorithm for zero-sum equilibrium
computation.  Notationally, we call $r^t = u^t - u^t\cdot x^t
e$ the \defword{instantaneous regret} and $R^t = \sum_{i=1}^t
r^i$ the \defword{cumulative regret} at time $t$.  Let $L = \sup_{u\in\mathcal
K} \Vert u \Vert_\infty$.

\begin{definition}[Regret-matching]
Choose $x^{t+1} \propto (R^t)_+$.  
\end{definition}
Here, $(x)_+ = \max\{0,x\}$.  
\begin{theorem}[from \cite{hart00}]
If $x^t$'s are chosen using regret-matching, then the external regret is no
more than $L\sqrt{NT}$.
\end{theorem}

The second algorithm is Hedge~\cite{freund97}. It is also known as the
exponentially-weighted forecaster, exponential weights, or weighted
majority~\cite{littlestone94}.  
\begin{definition}[Hedge]
Choose $x^{t+1} \propto \exp(\eta R^t)$.
\end{definition}
\begin{theorem}[from \cite{freund97}]
If $x^t$'s are chosen using Hedge with $\eta = \sqrt{2\log(N) / T}/L$,
then the external regret is no more than $L\sqrt{2T\log N}$.
\end{theorem}

For equilibrium-finding, the regret-matching algorithm of \citeauthor{hart00}
is common place.  Though Hedge has a slightly better theoretical bound than
regret-matching, it is computationally more expensive due to $\exp$ and
choosing $\eta$ can be tricky in practice.  As we will see, the use of Hedge
here leads to an insightful connection between the two equilibrium-finding
approaches in question.  We show the counterfactual regret update
in Algorithm~\ref{alg:cfr}.

\begin{algorithm}[t]
\begin{algorithmic}
	\Function{$\text{updateRegret}_I$}{R, g}
	\For{$a\in A(I)$}
	  \For{$I'\in \text{child}(I,a)$}
	    \State{$u', R_{I'} \gets $ \Call{$\text{updateRegret}_{I'}$}{$R_{I'}, g_{I'}$}}
	    \State{$g_a \gets g_a + u'$}
	  \EndFor
	\EndFor
	\State{$x_a \propto \exp(R_{I,a})$}
	\For{$a\in A(I)$}
	  \State{$R_{I,a} \gets R_{I,a} + g_a - g\cdot x$}
	\EndFor
	\State \Return{$g\cdot x, R$}
      \EndFunction
      \Function{updateRegret}{R, g}
        \For{$I\in \text{child}(\Phi)$}
	  \State{$\_, R_{I} \gets $ \Call{$\text{updateRegret}_{I'}$}{$R_I,g_I$}}
        \EndFor
      \EndFunction
\end{algorithmic}
\caption{CFR Update with Hedge}
\label{alg:cfr}
\end{algorithm}

\section{Connection to Convex Optimization}

A different view from the no-regret learning approach is simply to write
the equilibrium computation as a non-smooth minimization.  This minimization
is convex, as both the objective and the set of realization plans are.
\begin{theorem}
Let $f(x) = \max_{y\in\Sigma_2} -x^TAy$ then $f$ is convex on $\Sigma_1$.
Furthermore, if $x^*\in\argmin_{x\in\Sigma_1} f(x)$ then $x^*$ is a minimax
optimal strategy for the row player.
\end{theorem}
\begin{proof}
First, let $y'\in\argmax_{y\in\Sigma_2} -x^TAy$.  Claim
$f'(x) = -Ay' \in \partial f(x)$.  For any $x'\in\Sigma_1$,
\begin{align*}
f(x') - f(x) & = \max_{y\in\Sigma_2} -x'\cdot Ay - \max_{y\in\Sigma_2} -x\cdot Ay\\
	     & \ge -x'\cdot Ay' - \max_{y\in\Sigma_2} -x\cdot Ay\\
	     & = -x'\cdot Ay' + x\cdot Ay'\\
	     & = (-Ay')\cdot(x' - x)
\end{align*}
This is an instantiation of Danskin's theorem~\cite{bertsekas99}. 
By the optimality of $x^*$, for any $\hat{x}\in\Sigma_1$:
\begin{align*}
x^*\cdot Ay^* = -f(x^*) & \ge -f(x') = \hat{x}\cdot A\hat{y}.
\end{align*}
where $y^*$ maximizes $-x^*\cdot Ay$ and $\hat{y}$ maximizes $-\hat{x}\cdot Ay$.
\end{proof}
Note, the subgradient computation is precisely $y$'s best response. 
The CFR-BR technique is exactly CFR applied to this non-smooth 
optimization~\cite{johanson12}.

As $f$ is convex, and we can efficiently evaluate its subgradient via a best
response calculation, we can use any subgradient method to find an
$\varepsilon$-equilibrium strategy.  Unfortunately, the most basic approach, the
projected subgradient method, requires a complicated and costly projection onto
the set of realization plans.  We can avoid this projection by employing the
proper Bregman divergence.  In particular, if $h : \mathcal D \rightarrow
\mathbb R$ is a strongly convex such that we can quickly solve the
minimization
\begin{align}
\argmin_{x\in\mathcal D} \;\;g \cdot x + h(x)
\end{align}
we say $h$ fits $\mathcal D$.  In these cases, we can often use $h$ in place of
the squared $l_2$ distance and avoid any projections.  

\citeauthor{hoda08} (\citeyear{hoda08}) describe a family of diverences, or
distance generating functions, for the set of realization plans.  They
construct their family of distance generating functions inductively.  One such
$h$ is as follows:
\begin{align*}
	h_{I}(x,y) & = \sum_{a\in A(I)} x_{I,a}\log x_{I,a} + & \forall I\in\mathcal I_i\\
		 & \sum_{I'\in \text{child}(I,a)} x_{I,a}\left[ h_{I'}(y_{I'}/x_{I,a}) \right] \\
	h(x) & = \sum_{I\in\text{child}(\phi)} h_I(x_I)
\end{align*}
A few things worth noting.  First, by $\text{child}(x)$ we mean the set of
information sets that are immediate children of sequence $x$.  Second, we
slightly abuse of notation above in that $h_I(x,y)$ depends on the immediate
realization weights, $x$, and all child weights $y$.  We denote the child
weights belonging to information set $I'$ as $y_{I'}$.  Second, due to perfect
recall, this recursion does bottom out; there are information sets with no
children and there are no cycles.  At a terminal information set, $h_I$ is the
negative entropy function.  The recursion makes use of the dilation or
perspective operator.

\citeauthor{gordon06} (\citeyear{gordon06}) introduces the same function in his
supplementary materials, but does not provide a closed-form solution to its
minimization.  The closed-form solution is shown in Algorithm~\ref{alg:prox}.
Note that this algorithm is the same as a best response calculation where we
replace the max operator with the softmax operator.  The normalization step
afterwards restores the sequence form constraints, {\em i.e.}, converts from a
behavioral strategy to a realization plan.

\begin{algorithm}[t]
\begin{algorithmic}
      \Function{minimize $h_I$}{u}
	\For{$a\in A(I)$}
	  \For{$I'\in \text{child}(I,a)$}
	    \State{$u', x_{I'} \gets $ \Call{minimize $h_{I'}$}{$u_{I'}$}}
	    \State{$u_{I,a} \gets u_{I,a} + u'$}
	  \EndFor
	\EndFor
	\State{$x_{I,a} \propto \exp(u_{I,a})$} \Comment{softmax instead of max}
	\State \Return{$u_I\cdot x, x$}
      \EndFunction
      \Function{$\text{normalize}_I$}{x, Z}
	\For{$a\in A(I)$}
	  \State{$x_{I,a} \gets x_{I,a} / Z$}
	  \For{$I'\in \text{child}(I,a)$}
	    \State \Call{$\text{normalize}_{I'}$}{$x_{I'}, x_{I,a}$}
	  \EndFor
	\EndFor
      \EndFunction
      \Function{minimize $h$}{g}
        \For{$I\in \text{child}(\Phi)$}
	  \State{$\_, x_{I} \gets $ \Call{minimize $h_{I'}$}{$-g_I$}}
	  \State \Call{$\text{normalize}_I$}{$x, 1$}
        \EndFor
	\State \Return{$x$}
      \EndFunction
\end{algorithmic}
\caption{Smoothed Best Response}
\label{alg:prox}
\end{algorithm}

The computational structure of Algorithm~\ref{alg:cfr} and \ref{alg:prox} are
identical.  CFR, too, must normalize when computing an equilibrium to properly
average the strategies.  Both use the softmax operator to on the expected
utility to define the current policy.  In particular, if $R = 0$, then the
computation is equivalent with the exception of the initial sign of $g$.  

\section{Dual Averaging}

Nesterov's dual averaging (\citeyear{nesterov09}) is a subgradient method, and
thus we can use to find equilibria.  As we shall see, it has close connections
to counterfactual regret when equipped with \citeauthor{hoda08} distance function. 

Dual averaging defines two sequences, $x^t$, the query points, and $g^t$, the
corresponding sequence of subgradients.  The averages $\bar{x}$ and $\bar{g}$
converge to a primal-dual solution.
\begin{definition}
Let $\beta^t>0$ be a sequence of step sizes and $h : \mathcal D
\rightarrow \mathbb R$ be a strongly convex distance generating function.
The sequence $x^t$ is given by
\begin{align}
\nonumber x^{t+1} & = \argmin_{x\in\mathcal D} \;\;\frac{1}{t}\sum_{i=1}^t\left[ f(x^i) + f'(x^i)\cdot(x - x^i)\right] + \beta^t h(x)\\
& = \argmin_{x\in\mathcal D} \;\;\bar{g}^t\cdot x + \beta^t h(x)
\end{align}
\end{definition}
The method convergences for step sizes $O(\sqrt{t})$, or for an appropriately
chosen constant step size when the number if iterations is known ahead of time.

Interestingly, Hedge and regret-matching over the simplex are operationally
equivalent to dual averaging.
\begin{theorem}
Hedge is equivalent to Dual Averaging over the simplex with $\sigma(x) =
\sum_{i=1}^n x_i\log x_i$, $\beta^t = 1/(t\eta)$ and $g^t = -u^t$.
\end{theorem}
\begin{proof}
The dual averaging minimization can be written as the convex conjugate of the
negative entropy~\cite{boyd04}:
\begin{align*}
\lefteqn{\min_{x\in\Delta} \;\;\bar{g}^t\cdot x + \beta^t\sum_{i=1}^n x_i\log x_i} \\
  & = -\beta^t \max_{x\in\Delta} \;\; -\bar{g}^t\cdot x/\beta^t - \sum_{i=1}^n x_i\log x_i \\
  & = -\beta^t \log\sum_{i=1}^n\exp\left(-\hat{g}_i / \beta^t \right) 
\end{align*}
The gradient of the conjugate minimizes the objective~\cite{rockafellar70},
\begin{align*}
x^{t+1} & \propto \exp\left(-\bar{g}^t / \beta^t\right) 
	       = \exp\left(-\eta\sum_{i=1}^t g^t\right) \\ 
	       & \propto \exp\left(\eta\sum_{i=1}^t u^t - u^t\cdot x^te\right) 
	       = \exp\left(\eta R^t\right) 
\end{align*}
The last step follows from $\exp(x) \propto \exp(b)\exp(x) = \exp(be + x)$ for
any vector $x$ and constant $b$.
\end{proof}
In particular, note subtracting the expected utility in the regret update does
not at all alter the iterates.  We need only accumulate counterfactual utility
when using Hedge.

\begin{theorem}
Regret-matching is equivalent to Dual Averaging over the simplex with
$\sigma(x) = \Vert x_+ \Vert_2^2 / 2$, $\beta^t = e^T R^t_+ / t$ and $g^t =
u^t\cdot x^te - u^t$.
\label{thm:rmda}
\end{theorem}
\begin{proof}
Consider the dual averaging minimization without the normalization constraint:
\begin{align*}
x^{t+1} & = \argmin_{x\ge 0} \;\;\bar{g}^t\cdot x + \beta^t \Vert x_+\Vert_2^2 /2 \\
	& = \frac{(-\bar{g}^t)_+}{\beta^t} = \frac{t(\frac{1}{t}\sum_{i=1}^t u^i - u^i\cdot x^ie)_+}{e^T R^t_+}\\
	& = \frac{R^t_+}{e^T R^t_+}
\end{align*}
Note that by construction $x^{t+1}$ sums to one, therefore the normalization constraint holds.
In order for dual averaging to converge, we need $1/ \beta^t \in O(\sqrt{T})$.
This follows from the no-regret bound on regret-matching, $e^T R^T_+ \le
L\sqrt{NT}$.
\end{proof}
Following from Theorem~\ref{thm:rmda}, we see that CFR with regret-matching is
dual averaging with a \citeauthor{hoda08}-style divergence built on $\Vert
x_+\Vert^2_2$.  Note that the step sizes must be chosen appropriately to avoid
projection.  That is, this divergence may not be appropriate for other gradient
methods that rely on more stringent step size choices.

Let us explicitly instantiate dual averaging for solving the convex-concave
equilibrium saddle-point optimization.
\begin{align*}
	x^{t+1} & = \argmin_{x\in\Sigma_1} \frac{1}{t}\sum_{i=1}^t -Ay^i + \beta^t h(x), \\
	& = \argmin_{x\in\Sigma_1} -A\bar{y}^t + \beta^t h(x), \\
	y^{t+1} & = \argmin_{y\in\Sigma_2} \frac{1}{t}\sum_{i=1}^t A^Tx^i + \beta^t h(y). \\
	& = \argmin_{y\in\Sigma_2} A^T\bar{x}^t + \beta^t h(y).
\end{align*}
In words, dual averaging applied to the saddle-point problem can be thought of
as fictitious play with a smoothed best response as opposed to an actual best
response~\cite{brown51}.  

Dual averaging and Hedge are operationally equivalent at terminal information
sets, that is, ones where all sequences have no children.  At a non-terminal
information set, dual averaging responds as if its current policy is played
against the opponent's average strategy in future decisions.  Counterfactual
regret minimization, on the other hand, plays against the opponent's current
policy.  The no-regret property of the algorithm guarantees that these two
quantities remain close.  In rough mathematical terms, we have $\sum_{t=1}^T
x^t\cdot A y^t \approx x^T\cdot \sum_{t=1}^T Ay^t$ within $O(\sqrt{T})$.
Conceptually, we can conclude that counterfactual regret minimization, too, is
roughly equivalent to smoothed fictitious play.

Operationally, the difference is counterfactual regret minimization propagates
{\em and} accumulates the expected utilities from the child information sets in
the regrets.  Dual averaging, in spirit, is lazy and re-propagates these
utilities on each iteration.  We note that this re-propagation is not
algorithmically necessary when we have sparse stochastic updates as the
expected value of an information set only changes if play goes through it.
That is, we can memoize and update this value in a sparse fashion.

This understanding of CFR hints at why it outperforms its bound in practice and
why unprincipled speedups may indeed be reasonable.  In particular, we can
achieve faster rates of convergence, $O(L/T)$ as opposed to $O(L/\sqrt{T})$,
when minimizing smooth functions with gradient descent and the appropriate step
size.  Two no-regret learners in self-play are essentially smoothing the
non-smooth objective for one and other.  The smooth objective itself is
changing from iteration to iteration, but this suggests we can choose more
aggressive step sizes than necessary.  Further evidence of this is that the
convergence behavior for CFR-BR, a minimization of a known non-smooth
objective, is exhibits more volatile behavior that is closer to the true CFR
regret bound.

Dual averaging with the \citeauthor{hoda08} divergence is itself a
no-regret learner over the set of realization plans~\cite{xiao10}.  The regret
bound itself is a function of the strong convexity parameter of the distance
function.  The bound on which appears to be quite loose.  The above analysis
suggests that it should be similar, or perhaps in some cases better, than the
counterfactual regret bound on external regret.  This is not shocking as
counterfactual regret minimization is agnostic to no-regret algorithm in use.

\section{Initialization with a Prior}

When computing an equilibrium or in an online setting, typically the initial
strategy is uniformly random.  Though the initial error does drop quite
rapidly, it is often the case that we have available good priors available.
Particularly in an online setting, it is preferable start with a good strategy
instead of essentially learning both how to play the game and how to exploit an
opponent at the same time.  From the optimization perspective, we now discuss
sound approaches.

First, let us investigate how to initialize the dual weights---the cumulative
counterfactual regrets.  In dual averaging, the dual weights to $x$, $\bar{g} =
A\bar{y}$, is the utilities to $x$ of the opponent's optimal strategy.  If we
have guesses at $x^*$ and $y^*$, we can use those to initialize the dual
weights.  This view of the dual weights is a simplification what is being done
by~\citeauthor{brown14} (\citeyear{brown14}), where counterfactual regret
minimization is started from a known good solution.  From the convex
optimization view-point this is immediate.

An appealing property of this is that the representation of the opponent's
policy need not be known.  That is, so long as we can play against it we can
estimate the dual weights.  In some cases, the variance of this estimate may be
quite poor.  With more structured representations, or domain knowledge it may
be possible to greatly improve the variance.  

It is quite common when considering convex optimization problems to {\em
recenter} the problem.  In particular, note that dual averaging starts with the
policy $x_0 = \argmin_{x\in\mathcal D} h(x)$.  By instead choosing $h'(x) =
\nabla h(x) - \nabla h(x')\cdot x$, we shift the starting policy to $x'$.  In
the case of the negative entropy over the simplex, this is equivalent to
instead using a Kullback-Leibler divergence.  Note, the step size schedule is
an important practical factor that needs to be considered more carefully.  In
particular, smaller step sizes keep iterates closer too the initial policy.

\section{Convergence of the Current Policy}

Some work has considered using CFR's final policy as opposed to the
average policy.  Experimentally, it is shown that the current policy
works quite well in practice despite the lack of bounds on its performance.

A large assortment of recent work on stochastic convex optimization has
considered this problem in depth exploring different averaging and step size
selection schemes.  When we view the problem from the convex optimization
viewpoint these results transfer without modification.  In particular, it has
been shown that using a step size decreasing like $1/\sqrt{t}$ will lead to
convergence of the current iterate for non-smooth optimization.  That is, if
the opponent plays a best response, like in CFR-BR, we need not store the
average policy reducing the memory requirement by 50\%.

\section{Acceleration and Sampling}

An important reason that no-regret algorithms have emerged as the dominant
approach for large-scale equilibrium computation is their amenability to a
various forms of sampling.  At a high level, by introducing stochasticity we
can drastically reduce the computation on each iteration by introducing
different types of sparsity while only marginally increasing the number of
necessary iterations.

\begin{theorem}[from \cite{cesa-bianchi06}]
Let the sequence $x^t$ be chosen according to a no-regret algorithm with regret
bound $\sqrt{CT}$ and let $\tilde{x}^t \sim x^t$.  For all $\delta\in(0,1)$,
with probability at least $1-\delta$ the regret of the sequence $\tilde{x}^t$
is bounded by $\sqrt{CT} + \sqrt{T/2\log1/\delta}$.
\end{theorem}

By sampling $\tilde{y}$ from $y$, now we choose $u^{t+1} = A\tilde{y}^t$.  That
is, $\tilde{y}^t$ is a standard basis vector and $u^{t+1}$ is just a column of
$A$~\cite{lanctot09}.  This has a number of important computational
consequences.  First, we no longer require a matrix-vector multiplication
reducing the complexity of each iteration to linear from quadratic.  In
addition to the asymptotic reduction in complexity, we improve the constants
since selecting a column of $A$ requires no arithmetic operations.  In fact, we
may not even need to store $u^{t+1}$ if we can index directly into $A$.

A second important computational improvement is we can now often use integral
numbers in place of floating-point numbers~\cite{gibson13}.  In particular,
note that $r^t = u^t - u^t\cdot \tilde{x}^t e$ is integral so long as $u^t$ is.
Furthermore, $\bar{x}$ can be represented as a vector of counts--the number of
times each action is sampled.  By using regret-matching, even computing $x^t$
can be done completely with integers as no floating-point math is required to
convert $R^T$ to $x^{t+1}$.

Another form of sampling is possible when nature participates in the
game~\cite{zinkevich08}.  For example, imagine that $A = \sum_{i=1}^p A_i$ and
we can implicitly access each $A_i$.  This is the case when nature rolls a die
or draws a card from a shuffled deck.  Instead of explicitly forming and
storing $A$, or examining each $A_i$ on every iteration, we can choose one
randomly.  That is, $\tilde{A}^t = pA_{i^t}$, where $i^t \sim
\text{Uniform}([p])$.  When we cannot store $A$ this form of sampling reduces
each iteration's asymptotic complexity from linear in $p$ to constant.

Nesterov's excessive gap technique (\citeyear{nesterov05}) cannot handle
stochastic gradients.  This is the primary reason that it and the
\citeauthor{hoda08} divergence fell out of favor for equilibrium computation.
As we note here, the divergence itself has nothing to do with the inability to
handle stochasticity.  It is a power tool enabling us to consider a wide
variety of gradient methods.  The stochastic mirror prox (SMP) algorithm of
\citeauthor{juditsky11} (\citeyear{juditsky11}) is an extension of the
extra-gradient method that achieves a rate of $O(L/T + \sigma/\sqrt{T})$ on
saddle-point problems like ours.  Here, $\sigma$ is the variance of the
gradient.  This rate is optimal.  Specifically, it enables us to trade off
between the variance of the sampling technique and the slower $1/\sqrt{T}$ rate
that CFR achieves.  Current sampling procedures favors low computation with
sparse updates.  It is unlikely that using them along side SMP will work well
out-of-the-box.  Further investigation is needed to determine if a successful
compromise can practically improve performance.

{\small
\bibliographystyle{aaai}
\bibliography{paper}
}

\end{document}